\documentclass{article}

\usepackage{arxiv}

\usepackage[utf8]{inputenc} 
\usepackage[T1]{fontenc}    
\usepackage{hyperref}       
\usepackage{url}            
\usepackage{booktabs}       
\usepackage{amsfonts}       
\usepackage{nicefrac}       
\usepackage{microtype}      
\usepackage{graphicx}
\usepackage{doi}

\usepackage{amsfonts}
\usepackage{amsthm,physics}
\usepackage{amsmath}
\usepackage{amscd}
\usepackage{t1enc}
\usepackage[mathscr]{eucal}
\usepackage{indentfirst}
\usepackage{graphicx}
\usepackage{graphics}
\usepackage{pict2e}
\usepackage{epic}
\numberwithin{equation}{section}
\usepackage{epstopdf}
\usepackage{amsmath,amsthm,verbatim,amssymb,amsfonts,amscd,graphicx}
\usepackage{mathtools}
\usepackage[dvipsnames]{xcolor}
\usepackage{authblk}
\usepackage{bm}
\usepackage{soul}

\usepackage[dvipsnames]{xcolor}
\usepackage[margin=3cm]{caption}
\usepackage{amsmath}
\usepackage{amsthm}
\usepackage{amssymb}
\usepackage{amscd}
\usepackage{mathtools}
\usepackage{hyperref}
\usepackage{subfigure}
\usepackage{stmaryrd}
\usepackage{booktabs}
\usepackage{mdframed}
\usepackage{upgreek}
\usepackage{mathtools} \mathtoolsset{showonlyrefs=true} \MakeRobust{\eqref}
\usepackage{listings}
\usepackage{color}
\usepackage{fullpage}
\usepackage{stackrel}
\usepackage{textcomp}
\usepackage{enumerate}
\usepackage{accents}
\usepackage{graphicx}
\usepackage{epstopdf}
\usepackage{placeins}
\usepackage{bbm}
\usepackage{bm}
\usepackage{accents}
\usepackage[mathscr]{eucal}
\usepackage[linesnumbered,boxed,titlenumbered,ruled,nofillcomment]{algorithm2e}
\usepackage[T1]{fontenc}
\usepackage{lmodern}
\usepackage{scalerel}
\usepackage{wrapfig}
\usepackage{algorithmic}
\usepackage{float}

\definecolor{lbcolor}{rgb}{0.9,0.9,0.9}
\theoremstyle{plain}
\newtheorem{theorem}{Theorem} \numberwithin{theorem}{section}

\theoremstyle{definition}

\definecolor{mydarkblue}{rgb}{0,0.08,0.45}
\hypersetup{ %
	colorlinks=false,
	linkcolor=mydarkblue,
	citecolor=mydarkblue,
	filecolor=mydarkblue,
	urlcolor=mydarkblue}

\newcommand{\Dc}[0]{\mathcal{D}}

\definecolor{mypink1}{rgb}{0.858, 0.188, 0.478}
\definecolor{mypink2}{RGB}{219, 48, 122}
\definecolor{mypink3}{cmyk}{0, 0.7808, 0.4429, 0.1412}
\definecolor{mygray}{gray}{0.6}

\newlength{\leftstackrelawd}
\newlength{\leftstackrelbwd}
\def\leftstackrel#1#2{\settowidth{\leftstackrelawd}%
{${{}^{#1}}$}\settowidth{\leftstackrelbwd}{$#2$}%
\addtolength{\leftstackrelawd}{-\leftstackrelbwd}%
\leavevmode\ifthenelse{\lengthtest{\leftstackrelawd>0pt}}%
{\kern-.5\leftstackrelawd}{}\mathrel{\mathop{#2}\limits^{#1}}}

\def\E{{\mathbb E}}

\newcommand{\KL}{\text{KL}}

\newcommand{\ESLB}{\text{ESLB}}

\DeclareMathSymbol{\mhyphen}{\mathord}{AMSa}{"39}

\title{Variational Entropy Search for Adjusting Expected Improvement}

\author{
Nuojin Cheng$^{1}$\quad Stephen Becker$^{1}$ \\
$^1$Department of Applied Mathematics, University of Colorado, Boulder\\
\texttt{Nuojin.Cheng,Stephen.Becker@colorado.edu}
}

\renewcommand{\shorttitle}{\textit{arXiv} Template}

\hypersetup{
pdftitle={A template for the arxiv style},
pdfsubject={q-bio.NC, q-bio.QM},
pdfauthor={David S.~Hippocampus, Elias D.~Striatum},
pdfkeywords={First keyword, Second keyword, More},
}

\begin{document}
\maketitle

\begin{abstract}
Bayesian optimization is a widely used technique for optimizing black-box functions, with Expected Improvement (EI) being the most commonly utilized acquisition function in this domain. While EI is often viewed as distinct from other information-theoretic acquisition functions, such as entropy search (ES) and max-value entropy search (MES), our work reveals that EI can be considered a special case of MES when approached through variational inference (VI). In this context, we have developed the Variational Entropy Search (VES) methodology and the VES-Gamma algorithm, which adapts EI by incorporating principles from information-theoretic concepts. The efficacy of VES-Gamma is demonstrated across a variety of test functions and read datasets, highlighting its theoretical and practical utilities in Bayesian optimization scenarios.
\end{abstract}

\begin{center}
    \textcolor{red}{This is a preliminary technical report. For a more comprehensive study, please refer to \href{https://arxiv.org/abs/2501.18756}{https://arxiv.org/abs/2501.18756}.}
\end{center}

\section{Introduction}
\label{sec:intro}
Bayesian methods have consistently been at the forefront of probabilistic modeling, recognized for their effectiveness in managing uncertainty and their adaptability in integrating prior knowledge. In this realm, Bayesian optimization stands out as a key technique for optimizing black-box functions, denoted as $f:\mathbb{X}\to\mathbb{R}$. This approach is particularly valuable in situations where data is scarce or expensive to obtain. The core of Bayesian optimization lies in determining the optimal approach for selecting subsequent sampling points, a process typically guided by an acquisition function $\alpha:\mathbb{X}\to\mathbb{R}$. Among the various acquisition functions, Expected Improvement (EI) \cite{mockus1998application}, Upper Confidence Bound (UCB) \cite{srinivas2010gaussian}, and Knowledge Gradient (KG) \cite{frazier2008knowledge} are commonly used. EI, in particular, is favored for its clear formulation, ease of computation, and effective performance. However, EI tends to prioritize exploitation over exploration \cite{de2021greed}, and it is not always clear when it is the best choice to use.

Entropy search, initially proposed in \cite{hennig2012entropy}, introduces a paradigm shift in this context of acquisition functions. In optimization tasks aimed at maximization, this method focuses on reducing the entropy in the posterior distribution of the function's maximum. The entropy-based criterion evaluates the informational value of potential new observations, directing the optimization process towards areas where the greatest reduction in uncertainty can be achieved. Recent advancements in information-theoretic Bayesian optimization have led to growing attention to the family of entropy search algorithms,  such as predictive entropy search (PES)\cite{hernandez2014predictive}, max-value entropy search (MES)\cite{wang2017max}, and joint entropy search (JES)\cite{tu2022joint,hvarfner2022joint}. These methods aim to select the next sampling point $\bm x$ by minimizing the differential entropy of the optimal target position or value. Despite the effectiveness of entropy search, its practical application is hindered by the absence of a closed-form solution; an example is presented in Figure~\ref{fig:x_y_star}. Approximating mutual information using methods like MCMC and expectation propagation increases computational demands, reducing its practicality.

\begin{figure}[!ht]
	\centering
	\includegraphics[width = 1.0\textwidth]{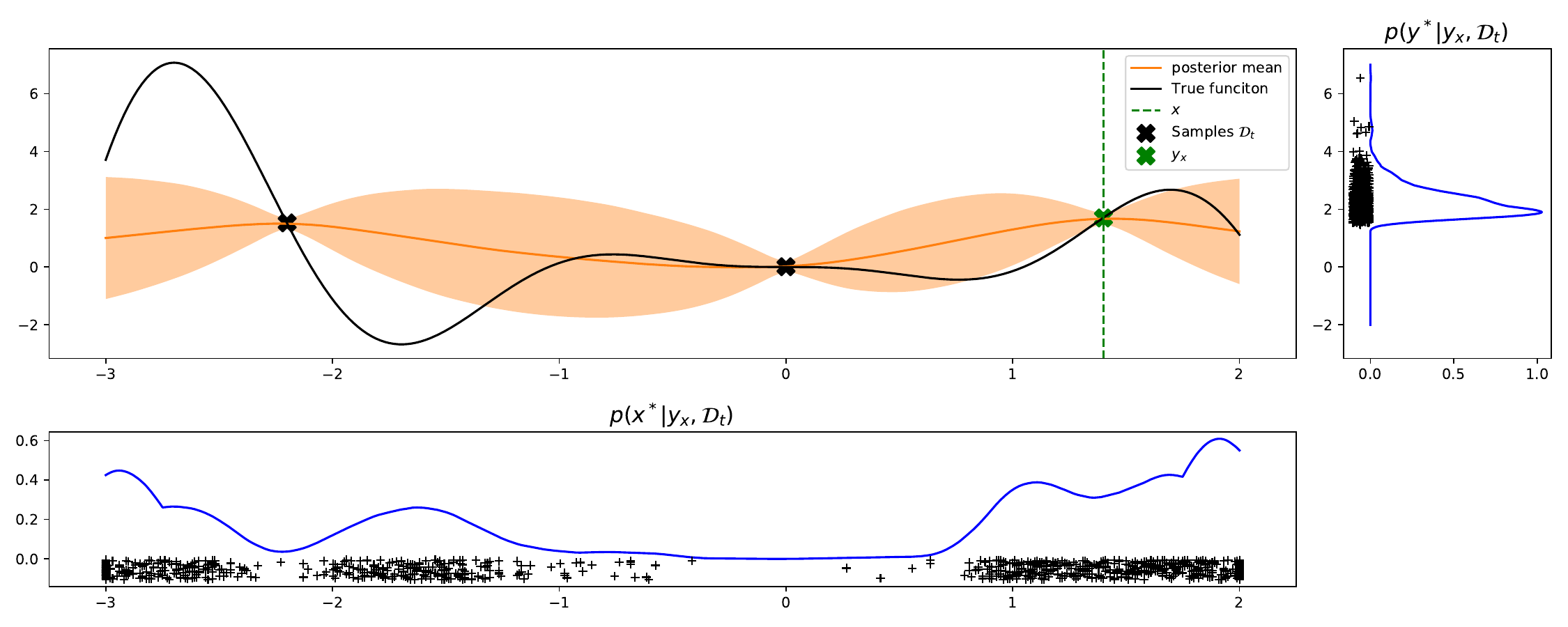}
	\caption{The distribution of $\bm x^*$ and $y^*$ approximated by kernel density estimation conditioned on given data points $\Dc_t$ and one potential evaluation $y_{\bm x}$ at $\bm x = 1.4$. The black crosses are sampled points from their corresponding distributions, which are drawn by sampling the posterior path $p(f\vert y_{\bm x}, \Dc_t)$. Entropy search methods are hindered by the lack of explicit formulation of $p(\bm x^*\vert y_{\bm x}, \Dc_t)$ or $p(y^*\vert y_{\bm x}, \Dc_t)$.}
	\label{fig:x_y_star}
\end{figure}

In \cite{hennig2022probabilistic}, acquisition functions are categorized into three types: value loss (VL), location-information loss (LIL), and value-information loss (VIL). Traditional acquisition functions like EI and KG fall under the VL category, which seeks to directly discover the highest possible value of the function under the Gaussian process prior. The VL functions typically involve calculating the expected maximum value of a function, parameterized by an input $\bm{x}$, and encourage selecting sampling points that are likely to yield the highest value. In contrast, LIL and VIL are similar in their goal to reduce the entropy of specific aspects of the function's maximum. The key difference between LIL and VIL is their focus: LIL aims to reduce the entropy of the location of the maximum point (e.g., ES \cite{hennig2012entropy} and PES \cite{hernandez2014predictive}), while VIL focuses on the value of the maximum point (e.g., MES \cite{wang2017max}). Our work suggests that the differences between these categories, especially between VL and VIL, may be less significant than previously thought. We show that EI can be viewed as a special case of MES, as outlined in Theorem~\ref{thm:ves-exp-ei}. This study explores the relationship between EI and MES and introduces a new method, termed Variational Entropy Search (VES), which combines elements from both EI and MES.

The primary principle of VES involves using variational inference (VI). VI is a method for approximating probability densities, particularly relevant in scenarios where direct computation of these densities is challenging due to their complexity. VI typically introduces a simpler, parameterized distribution, known as the variational distribution, to approximate the complex target distribution. This variational distribution is often chosen from a group of more manageable distributions, such as Gaussian, exponential, or Gamma distributions. In this study, we introduce variational distributions as alternatives to the sampling-based strategies that are typically used in other ES methods for approximating the ES acquisition function. The application of VI helps us view EI as a specific case of the MES approach. We then extend this concept to a broader range of variational distributions to improve the approximation and potentially address the issue of EI's tendency towards over-exploitation.

Our research makes two main contributions:
\begin{enumerate}
    \item We demonstrate that the relationship between EI and MES is stronger than previously recognized. By restricting variational distributions to all exponential distributions, we establish a direct link between EI and MES; see Theorem~\ref{thm:ves-exp-ei}.
    \item We expand the range of variational distributions from exponential to gamma distributions, leading to the development of a new acquisition function (VES-Gamma in Algorithm~\ref{alg:VES-gamma}). This function shows improved performance in Section~\ref{sec:experiment}.
\end{enumerate}

Our work is organized as follow. In Section~\ref{sec:bg}, we introduce two fundamental concepts for Bayesian optimization. Section~\ref{sec:ves} provides the motivations and details of the VES method, as well as the proposed VES-Gamma algorithm. Empirical experiments of VES-Gamma is compared with EI and MES on both test functions and real datasets in Section~\ref{sec:experiment}. We will discuss the conclusion in Section~\ref{sec:conclusion}. This study focuses solely on single-fidelity, scalar-output, maximization-driven, and noise-free Bayesian optimization.
\section{Background}
\label{sec:bg}
The goal of Bayesian optimization is to find the maximum of a black-box function $f:\mathbb{X}\to\mathbb{R}$, where $\mathbb{X}\subset\mathbb{R}^d$ is a defined compact subset. At each iteration $t$, our dataset $\Dc_t$ consists of pairs $\{(\bm x_i, y_{\bm x_i})\}_{i=1}^t$. Based on this dataset, we choose a new point $\bm x_{t+1}\in\mathbb{X}$ to evaluate, where $y_{\bm x_{t+1}}=f(\bm x_{t+1})$. We assume that these evaluations are without noise. To model the black-box function $f$, we typically use a Gaussian process regression model. Gaussian processes are briefly introduced in Section~\ref{ssec:gp}, and in Section~\ref{ssec:af}, we discuss acquisition functions, which help us decide the next point $\bm x_{t+1}$.

\subsection{Gaussian Process}
\label{ssec:gp}

A Gaussian process is a stochastic process used to model an unknown function. It is characterized by the property that any finite set of function evaluations follows a multivariate Gaussian distribution. A Gaussian process is uniquely determined by the current observations $\Dc_t=\{(\bm x_i, y_{\bm x_i})\}_{i=1}^t$ and the kernel function $\kappa(\bm x, \bm x')$. Given these, at time $t$, the predicted mean at a new point $\bm x$ is
\begin{equation}
    \mu_t(\bm x) = \bm\kappa_t(\bm x)^T (\bm K_t)^{-1} \bm y_t,
\end{equation}
and the predicted covariance between two points $\bm x$ and $\bm x'$ is
\begin{equation}
    \textup{Cov}_t(\bm x, \bm x') = \kappa(\bm x, \bm x') - \bm\kappa_t(\bm x)^T (\bm K_t)^{-1} \bm\kappa_t(\bm x'),
\end{equation}
where $[\bm\kappa_t(\bm x)]_i = \kappa(\bm x_i, \bm x)$, $[\bm y_t]_i = y_{\bm x_i}$, and $[\bm K_t]_{i, j} = \kappa(\bm x_i, \bm x_j)$; for more details, see \cite{rasmussen2006gaussian}. 

\subsection{Acquisition Functions}
\label{ssec:af}

While Gaussian processes provide a surrogate model for the unknown function $f$, choosing the next point $\bm x_{t+1}$ for evaluation is crucial. This choice is typically made using an acquisition function $\alpha:\mathbb{X}\to\mathbb{R}$, where we set $\bm x_{t+1} = \arg\max_{\bm x} \alpha(\bm x)$. We discuss two types of acquisition functions relevant to our study.

\paragraph{Expected Improvement}
As the most commonly used acquisition function, EI is formulated as 
\begin{equation}
\label{eqn:ei}
\begin{aligned}
    \alpha_\textup{EI}(\bm x) &= \E_{p(y_{\bm x}\vert\Dc_t)}[\max\{y_{\bm x} - y^*_t, 0\}]\\
    &= \E_{p(y_{\bm x}\vert\Dc_t)}[\max\{y_{\bm x}, y^*_t\}] - y^*_t\\
    &\equiv \E_{p(y_{\bm x}\vert\Dc_t)}[\max\{y_{\bm x}, y^*_t\}],
\end{aligned}
\end{equation}
where $y^*_t$ is the maximum function value in $\Dc_t$, and $\E_{p(\cdot)}$ denotes the expectation with respect to the density $p(\cdot)$. The last line's equivalence holds when we fix $\Dc_t$, meaning maximizing this formulation is equivalent to maximizing the EI function. Essentially, EI is one way to encourage selecting $\bm x$ such that its GP evaluation is higher than the current best observation.

\paragraph{Entropy Search}
ES is a family of acquisition functions designed to choose $\bm x$ whose evaluation reduces the uncertainty about the extreme points of the function. These extreme points can include the maximum position $\bm x^*$, the maximum function value $y^*$, or their joint distribution $(\bm x^*, y^*)$. The main idea behind ES is that the goal of optimization is to reduce uncertainty about these extreme points. Our focus is on MES, which aims to minimize the uncertainty about the GP's maximum value $y^*$. 
In informal notation, defining entropy $\mathbb{H}[y]=\mathbb{E}_{p(y)}[-\log p(y)]$ and conditional entropy as $\mathbb{H}[y | x ]=\mathbb{H}[x,y] - \mathbb{H}[x]$, then MES is 
\begin{equation}
\label{eqn:mes}
\begin{aligned}
    \alpha_\textup{MES}(\bm x) &= \mathbb{H}[y^*|\mathcal{D}_t] - \mathbb{E}_{p(y_{\bm x}\vert\Dc_t)}\left[\mathbb{H}[ y^*|\mathcal{D}_t,y_{\bm x}]\right]\\
    &\equiv \mathbb{E}_{p(y^*, y_{\bm x}\vert\Dc_t)}[\log(p(y^*\vert \Dc_t, y_{\bm x}))].
\end{aligned}
\end{equation}
This equivalence is valid because $\mathbb{H}[y^*\vert\Dc_t]$ is constant with respect to $\bm x$, and omitting this term does not affect the optimization of the acquisition function.

Comparing Equations~\eqref{eqn:ei} and \eqref{eqn:mes}, we note two main differences between these acquisition functions:
\begin{enumerate}
    \item EI calculates an expectation over $p(y_{\bm x}\vert\Dc_t)$, while MES involves an expectation over $p(y^*, y_{\bm x}\vert\Dc_t)$, a more complex density function that is harder to sample from;
    \item EI has a clear formulation (since $p(y_{\bm x}\vert\Dc_t)$ is known), whereas MES does not.
\end{enumerate}
These differences make EI more practical, leading to its widespread use in Bayesian optimization tasks. We aim to solve the second point, which involves the VI method.

\section{Variational Entropy Search}
\label{sec:ves}
In this section, we establish the Variational Entropy Search (VES) approach and show that EI is a specific instance of MES within this framework. To address the implicit nature of $p(y^*, y_{\bm x}\vert\Dc_t)$ in Equation~\eqref{eqn:mes}, the original MES work \cite{wang2017max} utilizes the symmetry of mutual information, as first proposed by \cite{hernandez2014predictive}, and transforms the MES acquisition function in Equation~\eqref{eqn:mes} into
\begin{equation}
\label{eqn:pmes}
\begin{aligned}
\alpha_\textup{MES}(\bm x) 
&= \mathbb{H}[y^*|\mathcal{D}_t] - \mathbb{E}_{p(y_{\bm x}\vert\Dc_t)}\left[\mathbb{H}[ y^*|\mathcal{D}_t,y_{\bm x}]\right]\\
&= \mathbb{I}[y^*; y_{\bm x}\vert\Dc_t]\\
&= \underbrace{\mathbb{H}[y_{\bm x}|\mathcal{D}_t]}_\textup{closed-form} - \mathbb{E}_{p(y^*\vert\Dc_t)}\underbrace{\left[\mathbb{H}[ y_{\bm x}|\mathcal{D}_t,y^*]\right]}_\textup{non-closed-form},
\end{aligned}
\end{equation}
where $\mathbb{I}[\,\cdot\,;\,\cdot\,]$ denotes mutual information. This transformation makes the acquisition function $\alpha_\textup{MES}$ more manageable. However, the density $p(y_{\bm x}|\mathcal{D}_t,y^*)$ remains unknown. The authors suggest approximating it as a truncated Gaussian, upper bounded by $y^*$, which simplifies the calculation but lacks a clearly formulated approximation error.

Alternatively, we propose approximating the unknown distribution $p(y^*|\mathcal{D}_t,y_{\bm x})$ in Equation~\eqref{eqn:mes} using the Variational Inference (VI) approach with a predefined family of densities. This approach trades accuracy for a closed-form expression and provides a formulation for the error induced by this approximation. Our result is presented in the following theorem:
\begin{theorem}
\label{thm:eslb}
The MES acquisition function in Equation~\eqref{eqn:mes} adheres to the Barber-Agakov (BA) bound as proposed in \cite{agakov2004algorithm} and can be lower bounded as follows,
\begin{equation}
\label{eqn:eslb}
\begin{aligned}
\alpha_\textup{MES}(\bm x) &= \mathbb{H}[y^*\vert\Dc_t] - \E_{p(y_{\bm x}\vert\Dc_t)}\mathbb{H}[y^*\vert \Dc_t, y_{\bm x}]\\
&= \mathbb{H}[y^*\vert\Dc_t] + \mathbb{E}_{p(y^*,y_{\bm x}\vert\Dc_t)}[\log(q(y^*\vert \Dc_t, y_{\bm x}))] + \mathbb{E}_{p(y_{\bm x}\vert\Dc_t)}[\KL\big(p(y^*\vert \Dc_t, y_{\bm x})\|q(y^*\vert \Dc_t, y_{\bm x})\big)]\\
&\geq \mathbb{H}[y^*\vert\Dc_t] + \mathbb{E}_{p(y^*,y_{\bm x}\vert\Dc_t)}[\log(q(y^*\vert\Dc_t,y_{\bm x}))],
\end{aligned}
\end{equation}
where $q(y^*\vert\Dc_t,y_{\bm x})$ is any chosen density function and $\KL(\cdot\|\cdot)$ represents the Kullback-Leibler divergence. The inequality is tight if and only if $\mathbb{E}_{p(y_{\bm x}\vert\Dc_t)}[\KL\big(p(y^*\vert \Dc_t, y_{\bm x})\|q(y^*\vert \Dc_t, y_{\bm x})\big)] = 0$.
\end{theorem}
The proof of Theorem~\ref{thm:eslb} is provided in Appendix~\ref{apdx:ba-lb}. As indicated in Equation~\eqref{eqn:eslb}, the error we aim to minimize is $\mathbb{E}_{p(y_{\bm x}\vert\Dc_t)}[\KL\big(p(y^*\vert \Dc_t, y_{\bm x})\|q(y^*\vert \Dc_t, y_{\bm x})\big)]$, which depends on $q$ and $\bm x$. When $\bm x$ is fixed, minimizing this error with respect to $q$ is equivalent to maximizing the remaining terms. Given that $\mathbb{H}[y^*\vert\Dc_t]$ is independent of both $q$ and $\bm x$, we omit it and define the remaining term as the Entropy Search Lower Bound (ESLB)
\begin{equation}
\label{eq:eslb}
    \ESLB(q,\bm x) \coloneqq \mathbb{E}_{p(y^*,y_{\bm x}\vert \Dc_t)}[\log(q(y^*\vert \Dc_t, y_{\bm x}))].
\end{equation}
Thus, maximizing $\alpha_\textup{MES}(\bm x)$ and minimizing the KL term can be combined into one objective: maximizing $\ESLB(q, \bm x)$. A typical approach in VI is to restrict the density $q$ to a set of potential options $\mathcal{Q}$. By parameterizing $q$ within $\mathcal{Q}$, the problem becomes tractable by solving $q$ and $\bm x$ iteratively. This process is detailed in Algorithm~\ref{alg:VES}.

\begin{algorithm}[ht]
		\DontPrintSemicolon
		\caption{Variational Entropy Search (VES) Framework\label{alg:VES}}
		\KwIn{Sample set $\mathcal{D}_t$, variational family $\mathcal{Q}$}
		\KwOut{acquisition maximizer $\bm x$}
		\begin{algorithmic}[1]
                \STATE initialize $\bm x^{(0)}$ 
			\FOR{$n = 1:N$}{
			    \STATE update variational posterior $q^{(n)}(y^*) \leftarrow \arg\max_{q\in\mathcal{Q}} \ESLB(q,\bm x^{(n-1)})$;
			    \STATE update $\bm x^{(n)} \leftarrow \arg\max_{\bm x} \ESLB(q^{(n)},\bm x)$;
			}
			\ENDFOR
			\STATE return $\bm x^{(N)}$
		\end{algorithmic}
\end{algorithm}
\paragraph{Connection to ELBO}
The concept of the ESLB in our approach bears similarities to the Evidence Lower Bound (ELBO), a concept frequently utilized in the machine learning community \cite{paisley2012variational, hoffman2013stochastic, kingma2013auto}. ELBO is employed in VI to approximately maximize the log-likelihood $\log p(\tilde{\bm x})$ with a latent variable $\bm z$. The log-likelihood is expressed as
\begin{equation}
\label{eqn:elbo}
\begin{aligned}
    \log p(\tilde{\bm x}) &= \KL(q(\bm z)\|p(\bm z\vert \tilde{\bm x}))+\E_{q(z)}\left[\log\left(\frac{p(\tilde{\bm x}\vert \bm z)p(\bm z)}{q(\bm z)}\right)\right]\\
         &\geq \E_{q(\bm z)}\left[\log\left(\frac{p(\tilde{\bm x}\vert \bm z)p(\bm z)}{q(\bm z)}\right)\right]=:\text{ELBO}(q(\bm z), p(\tilde{\bm x}\vert\bm z)),
\end{aligned}
\end{equation}
where $p(\bm z)$ is the prior distribution. In VI, the goal is to maximize $\log p(\tilde{\bm x})$ while minimizing the error term $\KL(q(\bm z)\|p(\bm z\vert \tilde{\bm x}))$. This is achieved by maximizing the ELBO as shown in Equation~\eqref{eqn:elbo}. The typical method involves iteratively optimizing the conditional probability $p(\tilde{\bm x}\vert \bm z)$ and the variational approximation $q(\bm z)$. Our ESLB approach follows a similar strategy. The shared properties between ELBO and our ESLB method are outlined in Table~\ref{tab:elbo-eslb}.

\begin{table}[ht]
\centering
\caption{Comparison of key aspects between the ELBO and ESLB approaches. \label{tab:elbo-eslb}}
\begin{tabular}{lcc}
\toprule
\textbf{Aspect} & \textbf{ELBO Approach} & \textbf{ESLB Approach} \\ 
\midrule
Primary Variable & $p(\tilde{\bm x}\vert\bm z)$ & $\bm x$ \\ 
Variational Variable & $q(\bm z)$ & $q(y^*\vert y_{\bm x}, \Dc_t)$ \\ 
Lower Bound Formulation & ELBO$\left(q(\bm z), p(\tilde{\bm x}\vert\bm z)\right)$ & ESLB$(q, \bm x)$ \\ 
\bottomrule
\end{tabular}
\end{table}

\paragraph{Connection to EI}
To establish a connection between the VES approach and the EI acquisition function, we define $\mathcal{Q}$ as the set of all exponential distributions, $\mathcal{Q}_\textup{exp}$, with the support lower bounded by $\max\{y_{\bm x},y^*_t\}$, where $y^*_t$ denotes the maximum function evaluation from $\Dc_t$. The variational density function $q(y^*\vert y_{\bm x}, \Dc_t)$, parameterized by $\lambda > 0$, is then expressed as
\begin{align}
\label{eq:eslb-exp}
    q(y^*\vert y_{\bm x}, \Dc_t) = \lambda \exp (-\lambda(y^* - \max\{y_{\bm x}, y^*_t\}))\mathbbm{1}_{y^*\geq\max\{y_{\bm x},y^*_t\}}.
\end{align}
In situations where there is no noise in the observations, the indicator function $\mathbbm{1}_{y^*\geq\max\{y_{\bm x},y^*_t\}}$ is always equal to one and can be omitted. We present the following theorem to describe the relationship between the ESLB with an exponential distribution and the EI acquisition function.

\begin{theorem}
\label{thm:ves-exp-ei}
    When the family $\mathcal{Q}_\textup{exp}$ is selected as in Equation~\eqref{eq:eslb-exp} and function evaluation is noise-free, the output of Algorithm~\ref{alg:VES} is equivalent to the output of the EI method.
\end{theorem}

\begin{proof}
    By restricting the variational distributions to exponential distributions, we slightly abuse the input notations of $\ESLB$ in Equation~\eqref{eq:eslb} and define:
    \begin{equation}
    \label{eqn:eslb-exp}
    \begin{aligned}
        \ESLB(\lambda, \bm x) &= \mathbb{E}_{p(y^*,y_{\bm x}\vert \Dc_t)}\left[\log\left(\lambda\exp\left(-\lambda (y^* - \max\{y_{\bm x},\Dc_t\})\right)\right)\right]\\
        &= \log\lambda -\lambda \mathbb{E}_{p(y^*,y_{\bm x}\vert \Dc_t)}\left[\left(y^* - \max\{y_{\bm x},y^*_t\}\right)\right]\\
        &= \log \lambda - \lambda \underbrace{\mathbb{E}_{p(y^*\vert \Dc_t)}[y^*]}_\textup{constant} + \lambda \underbrace{\mathbb{E}_{p(y_{\bm x}\vert \Dc_t)}\left[\max\{y_{\bm x},\Dc_t\}\right]}_\textup{EI acquisition}.
    \end{aligned}
    \end{equation}

Beginning with an arbitrary initial value $\bm x^{(0)}$, we determine the corresponding parameter
\begin{equation}
\label{eq:solve-lambda}
\lambda^{(1)} = \mathbb{E}_{p(y^*,y_{\bm x^{(0)}}\vert \Dc_t)}\left[\left(y^* - \max\{y_{\bm x^{(0)}},y^*_t\}\right)\right]^{-1},
\end{equation}
which is derived by taking the directive of Equation~\eqref{eqn:eslb-exp} and letting it equal to zero. With $\lambda$ fixed, $\ESLB(\lambda^{(1)}, \bm x)$ produces the same result as the EI acquisition function in Equation~\eqref{eqn:ei} 
. We then compute $\lambda^{(2)}$ based on $\bm x^{(1)}$ following Equation~\eqref{eq:solve-lambda}. Regardless of the specific value of $\lambda$, the $\ESLB$ function consistently yields the same result, $\bm x^{(1)}$. This consistency ensures that the VES iteration process converges in a single step. The final outcome, represented as $(\bm x^{(1)}, \lambda^{(2)})$, indicates that the corresponding $q(y^* | y_{\bm x}, \Dc_t)$ is the closest approximation to $p(y^* | y_{\bm x}, \Dc_t)$ within $\mathcal{Q}_\textup{exp}$.
\end{proof}

\begin{figure}[!ht]
	\centering
	\includegraphics[width = 0.6\textwidth]{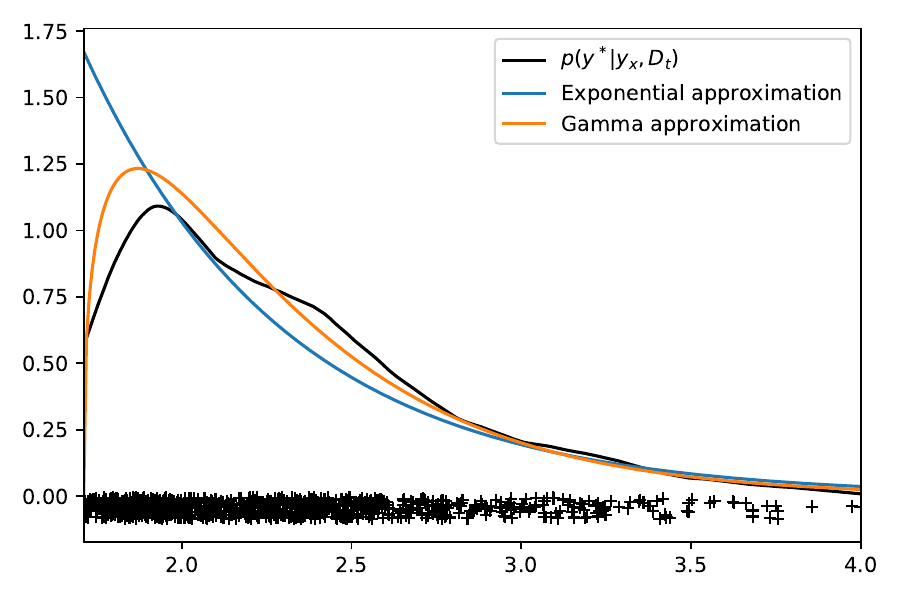}
	\caption{The approximation of $p(y^*\vert y_{\bm x}, \Dc_t)$ (black) in the right side of Figure~\ref{fig:x_y_star} using moment matching within exponential distributions (blue) and Gamma distributions (orange) on the domain $[1.71, 4.00]$.}
	\label{fig:y_star}
\end{figure}

Theorem~\ref{thm:ves-exp-ei} demonstrates that the EI can be viewed as a specific instance of MES when the variational approximation of $p(y^*\vert y_{\bm x}, \Dc_t)$ is limited to exponential distributions. This insight provides EI, a widely used acquisition function, with a new information-theoretical interpretation. However, the scope of variational distributions confined to exponential forms is somewhat narrow. As illustrated in Figure~\ref{fig:y_star}, the density of $p(y^*\vert y_{\bm x},\Dc_t)$ does not align well with the characteristics of an exponential distribution. Notably, the density of $y^* | y_{\bm x}, \Dc_t$ in the figure is not monotonic; it exhibits a peak before declining towards $\max\{y_{\bm x}, y^*_t\}$ ($\approx 1.71$). This pattern exceeds the representational limits of exponential distributions.

Therefore, there is a compelling reason to enrich the set of distributions $\mathcal{Q}$. By expanding this set, we can introduce greater flexibility into our variational approximation, leading to a more effective reduction in the entropy of $y^* | y_{\bm x}, \Dc_t$. A better approximation by introducing the Gamma distribution in shown in Figure~\ref{fig:y_star}. In the subsequent section, we will explore how to modify the VES framework to include the Gamma distribution, thereby extending our approach to a more general case that goes beyond the exponential distribution.

\paragraph{VES with Gamma Distribution}
In the context of the VES framework, when we define the family $\mathcal{Q}$ as the (shifted) Gamma distributions parameterized by $k,\beta>0$ with its support lower bounded by $\max\{y_{\bm x},y^*_t\}$, the variational density transforms into 
\begin{equation}
\label{eq:gamma-posterior}
    q(y^*\vert y_{\bm x}, \Dc_t) = \frac{\beta^k}{\Gamma(k)} \left(y^* - \max\{y_{\bm x},\Dc_t\}\right)^{k - 1} \exp \left(-\beta(y^* - \max\{y_{\bm x}, \Dc_t\})\right)\mathbbm{1}_{y^*\geq\max\{y_{\bm x},y^*_t\}},
\end{equation}
where $\Gamma(\cdot)$ denotes the Gamma function.
The noise-free assumption lets the indicator function omitted in Equation~\eqref{eq:gamma-posterior}, and the ESLB is reformulated as
\begin{equation}
\label{eqn:eslb-gamma}
\begin{aligned}
    &\ESLB(k, \beta, \bm x)\\
    &= k\log \beta - \log\Gamma(k) + (k - 1)\mathbb{E}_{p(y^*,y_{\bm x}\vert \Dc_t)}\left[\log\left(y^* - \max\{y_{\bm  x}, y^*_t\}\right)\right] - \beta\mathbb{E}_{p(y^*,y_{\bm x}\vert \Dc_t)}\left[\left(y^* - \max\{y_{\bm x},y^*_t\}\right)\right]\\
    &= k\log \beta - \log\Gamma(k) + (k - 1)\underbrace{\mathbb{E}_{p(y^*,y_{\bm x}\vert \Dc_t)}\left[\log\left(y^* - \max\{y_{\bm  x}, y^*_t\}\right)\right]}_\textup{``anti-EI''} - \beta \underbrace{\E_{p(y^*\vert \Dc_t)}[y^*]}_\textup{constant} + \beta \underbrace{\E_{p(y_{\bm x}\vert\Dc_t)}[\max\{y_{\bm x}, y^*_t\}]}_\textup{EI}.
\end{aligned}
\end{equation}

The ESLB in Equation~\eqref{eqn:eslb-gamma} becomes the main objective in our VES-Gamma Algorithm. We identify two key terms related to the value of $\bm x$: the EI term and the ``anti-EI'' term. The EI term is equivalent to $\alpha_\textup{EI}$ in Equation~\eqref{eqn:ei}. The ``anti-EI'' term, as implied by its name, acts in contrast to the EI term. Maximizing this term with respect to $\bm x$ encourages the selection of $\bm x$ values that result in lower GP evaluations, serving as a regulatory mechanism to balance the potential over-exploitation tendency of EI.

Notably, when the parameter $k$ is set to 1, the Gamma distribution reverts to an exponential distribution, leading to the elimination of the ``anti-EI'' term. As a result, the ESLB in Equation~\eqref{eqn:eslb-gamma} becomes equivalent to the earlier ESLB in Equation~\eqref{eqn:eslb-exp}. The value of $k$ is crucial, indicating the degree of regularization and reflecting the level of reliance the VES-Gamma method places on EI in a specific context. A higher value of $k$ suggests less trust in EI, while a $k$ value less than 1 indicates a strong alignment of VES-Gamma with EI outcomes.

The global maximum of the ESLB with respect to $\beta$ and $k$ exists, as can be demonstrated through derivative analysis, and can be approximated without the need for an optimizer. To determine the optimal $\beta$ and $k$ for a fixed $\bm x$, we derive the following equations:
\begin{align}\label{eq:ves-gamma-k}
    \frac{\partial}{\partial k}\ESLB(k,\beta,\bm x) &= \log\beta - \frac{\partial \log \Gamma(k)}{\partial k} + \mathbb{E}_{p(y^*,y_{\bm x}\vert \Dc_t)}\left[\log\left(y^* - \max\{y_{\bm  x}, y^*_t\}\right)\right] = 0,\\
    \frac{\partial}{\partial \beta}\ESLB(k,\beta,\bm x) &= \frac{k}{\beta} - \mathbb{E}_{p(y^*,y_{\bm x}\vert \Dc_t)}\left[\left(y^* - \max\{y_{\bm x},y^*_t\}\right)\right] = 0.\label{eq:ves-gamma-beta}
\end{align}
The solution for $k$ can be implicitly expressed as 
\begin{equation}
\label{eq:k-solve}
    \log k - \frac{\partial \log \Gamma(k)}{\partial k} = \log \mathbb{E}_{p(y^*,y_{\bm x}\vert \Dc_t)}\left[\left(y^* - \max\{y_{\bm x},y^*_t\}\right)\right] - \mathbb{E}_{p(y^*,y_{\bm x}\vert \Dc_t)}\left[\log\left(y^* - \max\{y_{\bm  x}, y^*_t\}\right)\right]. 
\end{equation}

The term $\partial \log \Gamma(k)/\partial k$, known as the digamma function \cite{abramowitz1988handbook}, can be efficiently approximated as a series. Thus, the root of Equation~\eqref{eq:k-solve} can be effectively estimated using a trivial root finding method. With the root $k^*$ determined, the value of $\beta$ is given by
\begin{equation}
\label{eq:beta-solve}
\beta^* = k^*/\mathbb{E}_{p(y^*,y_{\bm x}\vert \Dc_t)}\left[\left(y^* - \max\{y_{\bm x},y^*_t\}\right)\right].
\end{equation}
The VES-Gamma Algorithm, based on these principles, is detailed in Algorithm~\ref{alg:VES-gamma}. A 1D example of VES compared with EI and MES is shown in Figure~\ref{fig:1d-res}.
\begin{algorithm}[ht]
		\DontPrintSemicolon
		\caption{VES-Gamma\label{alg:VES-gamma}}
		\KwIn{Sample set $\mathcal{D}_t$}
		\KwOut{acquisition function maximizer $\bm x$}
		\begin{algorithmic}[1]
                \STATE initialize $\bm x^{(0)}$ 
			\FOR{$n = 1:N$}{
                    \STATE Evaluate $\mathbb{E}_{p(y^*,y_{\bm x}\vert \Dc_t)}\left[\left(y^* - \max\{y_{\bm x},y^*_t\}\right)\right]$ and $\mathbb{E}_{p(y^*,y_{\bm x}\vert \Dc_t)}\left[\log\left(y^* - \max\{y_{\bm  x}, y^*_t\}\right)\right]$ by sampling $p(y^*,y_{\bm x}\vert \Dc_t)$ given $\bm x = \bm x^{(n-1)}$;
                    \STATE Solve $k^{(n)}$ from Equation~\eqref{eq:k-solve};
                    \STATE Solve $\beta^{(n)}$ from Equation~\eqref{eq:beta-solve};
			    \STATE Update $\bm x^{(n)} \leftarrow \arg\max_{\bm x} \ESLB(k^{(n)}, \beta^{(n)}, \bm x)$ defined in Equation~\eqref{eqn:eslb-gamma}\label{algs};
			}
			\ENDFOR
			\STATE return $\bm x^{(N)}$
		\end{algorithmic}
\end{algorithm}

\begin{figure}[!ht]
	\centering
	\includegraphics[width = 1.0\textwidth]{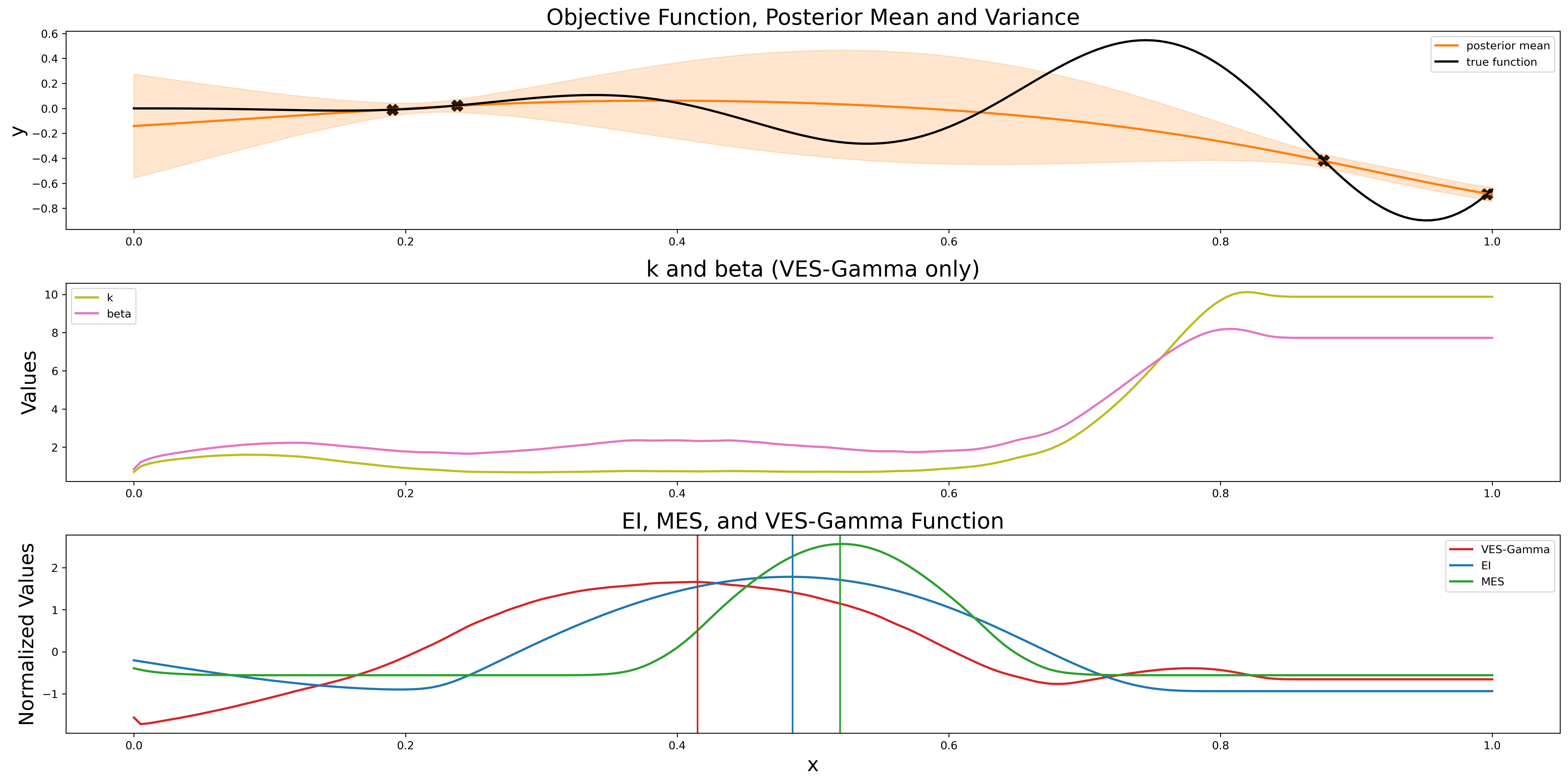}
	\caption{Comparison of EI, MES, and our proposed VES-Gamma (see Algorithm~\ref{alg:VES-gamma}) acquisition functions. The top figure shows the test function (black) and the Gaussian posterior (orange) with current observed points (black crosses). The middle figure displays the values of $k$ (olive) and $\beta$ (pink) calculated in Equation~\eqref{eq:k-solve} and Equation~\eqref{eq:beta-solve}, reflecting the reliance of VES-Gamma on the EI acquisition function. The bottom figure compares the three acquisition functions and their maximizers. The acquisition function values are centered and re-scaled for comparison.}
	\label{fig:1d-res}
\end{figure}

\section{Experiments}
\label{sec:experiment}
In this section, we probe the empirical performance of VES-Gamma on a variety of tasks. We employ Gaussian process priors for the function $f$, utilizing a Matern 5/2 kernel defined as
\begin{equation}
\label{eq:matern52}
    \kappa_\textup{Matern52}(\bm x, \bm x') = \sigma_f^2 \left(1+\frac{\sqrt{5}\lVert \bm x - \bm x'\rVert}{\sigma_l}+\frac{5\lVert \bm x - \bm x'\rVert^2}{3\sigma_l^2}\right)\exp\left(-\frac{\sqrt{5}\lVert \bm x - \bm x'\rVert}{\sigma_l}\right),
\end{equation}
where $\sigma_l$ and $\sigma_f$ are the kernel parameters. For each iteration, we optimize these parameters by applying maximum log-likelihood estimation (MLE) given current observations $\Dc_t$. To evaluate performance, we use the (simple) log-regret metric $r(t)$, defined as
\begin{equation}
    r(t) \coloneqq \log\left(f^* - \max_{(\bm x_i, y_{\bm x_i})\in\Dc_t} y_{\bm x_i}\right),
\end{equation}
where $f^* \coloneqq \max_{\bm x\in \mathbb{X}}f(\bm x)$. In cases where $y_{\bm x_i}$ equals $f^*$, we set the log regret to $-16$, aligning with the logarithm of machine precision of double precision floating point. For comparative purposes, we include implementations of EI and MES using BoTorch \cite{balandat2020botorch}, with the same settings for the Gaussian process prior. We restrict the focus of Bayesian optimization on 2D problems since we use grid search to optimize  Step~\ref{algs} in Algorithm~\ref{alg:VES-gamma} at $101\times 101$ grid points, though we expect that practical implementations in other dimensions would work nearly as well. We utilize 1,024 path samples to approximate the expectations in Equations~\eqref{eq:k-solve} and \eqref{eq:beta-solve}. The evaluations of these functions are conducted without noise, as assumed in Section~\ref{sec:ves}. All the presented results are repeated 10 times to evaluate the average performance and their standard deviations.

\subsection{Test Functions} 

Our investigation began with a comparative analysis of different acquisition functions using three benchmark test functions: the Rosenbrock function in 2D \cite{rosenbrock1960automatic}, the Three Hump Camel function \cite{dixon1978optimization}, and Himmelblau's function \cite{himmelblau2018applied}. The Rosenbrock function is widely recognized for its effectiveness in evaluating optimizers, while the Three Hump Camel and Himmelblau's functions are characterized by their multi-modal nature. The results of these tests are displayed in Figure~\ref{fig:2d-res}, where we initiated the process with two observations and monitored the log regret values over a span of 100 steps.

\begin{figure}[!ht]
	\centering
	\includegraphics[width = 1.0\textwidth]{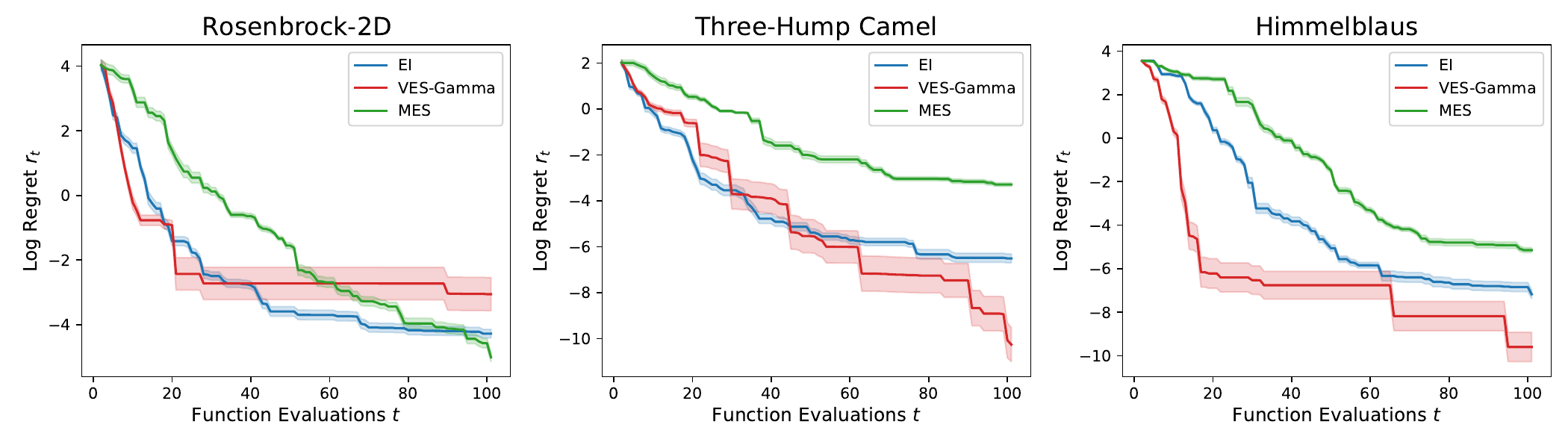}
	\caption{Performance comparison of VES, EI, and MES on three test functions, illustrating the mean and 0.1 standard deviations for log regret.}
	\label{fig:2d-res}
\end{figure}

As indicated in Figure~\ref{fig:2d-res}, the VES-Gamma method exhibits superior performance over the other acquisition functions, particularly on the Three-Hump Camel and Himmelblau's functions. This observation suggests that VES-Gamma may be more adept at handling multi-modal functions, given the multi-modal properties of these two test cases. However, it is also noted that the uncertainty associated with VES-Gamma increases in the latter stages of the experiment, implying that the choice of initial points, with its inherent randomness, might have a more pronounced impact on VES-Gamma compared to EI and MES. Overall, VES-Gamma shows promising potential in outperforming both EI and MES.

\subsection{Hyper-parameter Tuning with Real Datasets}

In addition to our evaluations using test functions, we extended our research to include hyper-parameter tuning on two real-world datasets, each tailored to a distinct type of problem. The first dataset, commonly referred to as the diabetes dataset \cite{efron2004least}, was utilized for a regression problem, whereas the second dataset, the iris dataset \cite{fisher1936use}, was applied to a classification problem. We employed the XGBoost algorithm as the solver. The hyper-parameter tuning process involved adjusting the learning rate within a range from 0 to 1 and the XGB-gamma value between 0 and 5. The XGB-gamma parameter is critical as it represents the minimum loss reduction necessary for additional partitioning on a leaf node of the tree. For the diabetes dataset, the objective function was the negative cross-validation mean squared errors, while for the iris dataset, it was the cross-validation accuracy, both dependent on the chosen values of the learning rate and XGB-gamma. The results of these hyper-parameter tuning experiments are illustrated in Figure~\ref{fig:real-case}, where we initiated the process with two observations and monitored the log regret values over a span of 45 steps. These findings suggest that the VES-Gamma method outperforms the other two acquisition functions in these specific hyperparameter tuning contexts.

\begin{figure}[htp]
    \centering
    \begin{minipage}{0.49\textwidth}
        \includegraphics[width=0.9\textwidth]{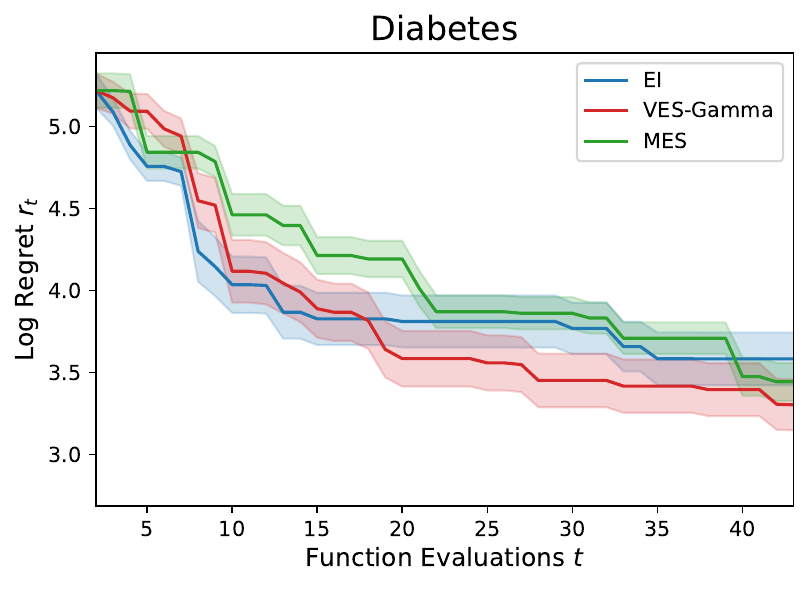}
    \end{minipage}\hfill
    \begin{minipage}{0.49\textwidth}
        \includegraphics[width=0.9\textwidth]{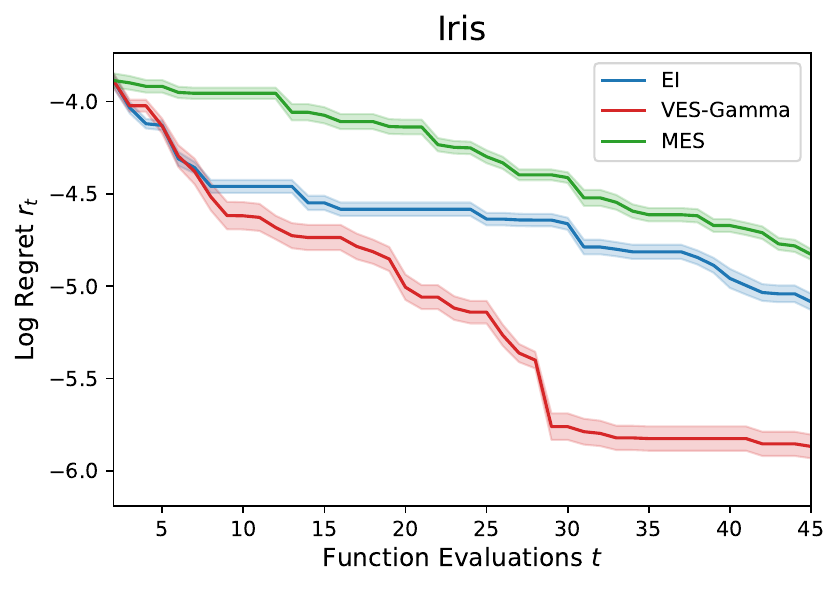}
    \end{minipage}
    \caption{Performance of VES, EI, and MES for XGBoost hyper-parameter tuning on the diabetes (left) and iris (right) datasets. The figure presents the mean and 0.1 standard deviation of ten trials with different random initial points.}
    \label{fig:real-case}
\end{figure}

\section{Conclusion}
\label{sec:conclusion}
In this paper, we highlighted the intrinsic connection between Expected Improvement (EI) and Max-value Entropy Search (MES). We have demonstrated that EI can be considered a specific case of MES. Furthermore, we introduced the concept of addressing the Entropy Search problem using Variational Inference within our Variational Entropy Search (VES) framework. The VES-Gamma algorithm, a key contribution of this work, has been shown to be effective in comparison to EI and MES across a variety of test functions.

While we have proposed using the parameter $k$ as a measure of confidence in the EI acquisition function, a comprehensive analysis of why a larger $k$ value indicates reduced reliance on EI has not been undertaken. Furthermore, our VES-Gamma algorithm extends the variational distributions only to Gamma distributions. This leaves open the possibility of exploring additional distributions. Future research could broaden the scope of variational distributions to encompass, for instance, chi-squared and generalized Gamma distributions. Such an expansion would provide further insights into how these methods might enhance the EI algorithm, offering new avenues for optimization in Bayesian frameworks.

\section*{Acknowledgements}
We thank Carl Hvarfner and Luigi Nardi for discussions about details in the entropy search method, as well as their support on the experiment coding.
N.\ Cheng was supported by the AFOSR awards FA9550-20-1-0138 with Dr.\ Fariba Fahroo as the program manager. 
S.\ Becker was supported by DOE awards DE-SC0023346 and DE-SC0022283.
The views expressed in the article do not necessarily represent the views of the AFOSR or the U.S.\ Government.

\bibliographystyle{abbrv}
\bibliography{references}  

\appendix
\section{Variational Lower Bound}
\label{apdx:ba-lb}
In this section we will show the proof of Theorem~\ref{thm:eslb}. 
\begin{proof}
\begin{equation}
\label{eq:eslb-pf}
\begin{aligned}
    \alpha_{\text{MES}}(\bm x) &= \mathbb{H}[y^*\vert\Dc_t] - \E_{p(y_{\bm x}\vert\Dc_t)}\mathbb{H}[y^*\vert \Dc_t, y_{\bm x}]\\
                          &= \mathbb{H}[y^*\vert\Dc_t] + \mathbb{E}_{p(y^*, y_{\bm x}\vert\Dc_t)}[\log(p(y^*\vert \Dc_t, y_{\bm x}))]\\
                          &= \mathbb{H}[y^*\vert\Dc_t] + \mathbb{E}_{p(y^*,y_{\bm x}\vert\Dc_t)}\bigg[\log(\frac{p(y^*\vert \Dc_t, y_{\bm x})q(y^*\vert \Dc_t, y_{\bm x})}{q(y^*\vert \Dc_t, y_{\bm x})})\bigg]\\
                          &= \mathbb{H}[y^*\vert\Dc_t] + \mathbb{E}_{p(y^*,y_{\bm x}\vert\Dc_t)}[\log(q(y^*\vert \Dc_t, y_{\bm x}))] + \mathbb{E}_{p(y_{\bm x}\vert\Dc_t)}[\KL\big(p(y^*\vert \Dc_t, y_{\bm x})\|q(y^*\vert \Dc_t, y_{\bm x})\big)]\\
                              &\geq \mathbb{H}[y^*\vert\Dc_t] + \mathbb{E}_{p(y^*,y_{\bm x}\vert\Dc_t)}[\log(q(y^*\vert\Dc_t,y_{\bm x}))].
\end{aligned}
\end{equation}
The inequality is tight if and only if $\mathbb{E}_{p(y_{\bm x}\vert\Dc_t)}[\KL\big(p(y^*\vert \Dc_t, y_{\bm x})\|q(y^*\vert \Dc_t, y_{\bm x})\big)] = 0$, which implies $p(y^*\vert \Dc_t, y_{\bm x}) = q(y^*\vert \Dc_t, y_{\bm x})$ for all $y_{\bm x}\vert\Dc_t$.
\end{proof}

\end{document}